\documentclass[preprint,11pt]{elsarticle}
\usepackage[a4paper,margin=2.5cm]{geometry}

\usepackage{amsmath,amssymb,amsfonts,amsthm}
\usepackage{mathrsfs}
\usepackage{graphicx}
\usepackage{booktabs}
\usepackage{algorithm}
\usepackage{algpseudocode}
\usepackage{xcolor}
\usepackage{hyperref}
\usepackage{url}
\usepackage{amsthm}

\newtheorem{Theorem}{Theorem}[section]

\newtheorem{Definition}[Theorem]{Definition}
\newtheorem{Notation}[Theorem]{Notation}
\newtheorem{Assumption}[Theorem]{Assumption}
\newtheorem{Proposition}[Theorem]{Proposition}
\newtheorem{Lemma}[Theorem]{Lemma}
\newtheorem{Corollary}[Theorem]{Corollary}

\newtheorem{Remark}[Theorem]{Remark}
\newtheorem{Example}[Theorem]{Example}

\newcommand{\tr}{\operatorname{tr}}
\newcommand{\rank}{\operatorname{rank}}
\newcommand{\R}{\mathbb{R}}
\newcommand{\diag}{\operatorname{diag}}

\journal{Linear Algebra and its Applications}

\begin{document}

\begin{frontmatter}



\title{Density-Matrix Spectral Embeddings for Categorical Data: Operator Structure and Stability}


\author[1]{Raquel Bosch-Romeu} 
\author[1]{Antonio Falcó}
\author[2]{José-Antonio Rodríguez-Gallego\corref{cor1}}
\ead{gallego@us.es}

\cortext[cor1]{Corresponding author}

\affiliation[1]{organization={ Departamento de Matem\'aticas, F\'{\i}sica y Ciencias Tecnol\'ogicas, Universidad Cardenal Herrera-CEU, CEU Universities},
            addressline={San Bartolom\'e 55}, 
            city={Alfara del Patriarca (Valencia)},
            postcode={46115}, 
            country={Spain}}
\affiliation[2]{organization={Department of Differential Equations \& Numerical Analysis},
            addressline={Avda. Reina Mercedes s/n}, 
            city={Seville},
            postcode={41012}, 
            country={Spain}}

\begin{abstract}
    We introduce a supervised dimensionality reduction methodology for categorical (and discretized mixed-type) data based on a density-matrix construction induced by class-conditional frequencies. 
    Given a labeled dataset encoded in a one-hot survey space, we assemble a frequency matrix whose columns aggregate feature occurrences within each class, and define a normalized Gram-type operator that satisfies the axioms of a density matrix. 
    The resulting representation admits an intrinsic rank bound controlled by the number of classes, enabling low-dimensional spectral embeddings via dominant eigenmodes. 
    Classification is performed in the reduced space through class-conditional kernel density estimation and a maximum-likelihood decision rule.
    We establish structural invariances, provide complexity estimates, and validate the approach on synthetic benchmarks probing high cardinality, sparsity, noise, and class imbalance.
\end{abstract}


\begin{highlights}
    \item A PSD density-operator for supervised representation of categorical data 
    \item Intrinsic rank bounded by number of classes 
    \item Low-dimensional embeddings obtained from dominant spectral subspace 
    \item A latent-space classifier is built by class-conditional kernel density estimation 
    \item Davis--Kahan perturbation and subspace-stability analysis with synthetic validation 
\end{highlights}


\end{frontmatter}



\section{Introduction}

High-dimensional categorical data arise naturally in many settings, including surveys, event logs, transactional records, and, more generally, problems in which qualitative descriptors dominate the feature space.
A standard encoding expands each categorical variable into a one-hot block and concatenates all blocks into a sparse binary vector.
Although this representation is faithful, it may lead to very large ambient dimensions when variables have many modalities, when rare categories are present, or when quantitative covariates are discretized.
The resulting data matrices are sparse and highly structured.
From a linear-algebraic viewpoint, such datasets are naturally summarized by contingency-type matrices and associated Gram operators, making spectral structure and low-rank approximations central to the analysis.

A classical and influential line of work treats categorical data through matrix factorizations of suitably normalized contingency tables.
Correspondence Analysis (CA) and its extensions, including Multiple Correspondence Analysis (MCA) and Discriminant Correspondence Analysis (DCA), derive geometric representations through singular value decompositions and weighted inner products, thereby extracting low-dimensional structure from categorical tables \cite{Greenacre1984,DCA}.
More broadly, spectral and kernel methods represent data through positive semidefinite operators whose eigenstructure encodes the relevant geometry; representative examples include kernel principal component analysis \cite{Scholkopf1998}, Laplacian eigenmaps \cite{BelkinNiyogi2003}, and spectral clustering \cite{NgJordanWeiss2002}.
These perspectives suggest that sparse one-hot categorical data may benefit from a carefully designed positive semidefinite operator whose dominant spectral directions capture class-relevant structure.

In this work, we propose a supervised operator construction tailored to labeled one-hot categorical data.
Starting from class-conditional frequency information, we construct a symmetric positive semidefinite, unit-trace matrix, which we interpret as a \emph{density operator} in the sense of quantum probability and quantum information theory \cite{NielsenChuang2000,Holevo2001}.
The normalization is not merely formal: after an entrywise square-root (amplitude) transform of class-conditional profiles, the construction induces a geometry closely related to the Bhattacharyya coefficient and the Hellinger distance between discrete distributions \cite{Bhattacharyya1943,Aherne1998}.
A key linear-algebraic consequence is an intrinsic low-rank bound: the resulting operator has rank at most the number of classes, and therefore its nonzero spectrum is fully determined by a $k\times k$ Gram matrix.
This yields a compact spectral embedding whose effective dimension is controlled by the number of classes rather than by the ambient one-hot dimension.

Beyond its basic structural properties (symmetry, positive semidefiniteness, unit trace, and rank control), the operator viewpoint enables a perturbation analysis that is central in matrix theory.
Small deviations in empirical class frequencies induce controlled perturbations of the operator, and, under a spectral-gap condition, the dominant invariant subspace (hence the associated embedding) is stable.
We formalize this through Davis--Kahan type $\sin\Theta$ bounds \cite{DavisKahan1970,YuWangSamworth2015} and complement them with explicit high-probability perturbation estimates under multinomial sampling in a blockwise categorical model.
These results provide a principled basis for comparing the proposed embedding with related spectral and low-rank constructions.

As an application of the spectral coordinates, we map samples into the dominant eigenspace and perform class assignment in the reduced space using class-conditional density estimation.
We employ kernel density estimation (KDE) in the Parzen--Rosenblatt framework \cite{Rosenblatt1956,Parzen1962}, with standard bandwidth selection considerations \cite{Silverman1986,Scott2015,WandJones1995}.
While the downstream classifier can be replaced by alternative decision rules, this choice emphasizes the probabilistic interpretation of the latent representation and keeps the computational cost governed by the reduced dimension.

\subsection{Contributions}

The main contributions of the paper are as follows:
\begin{itemize}
  \item \textbf{A supervised PSD operator for one-hot categorical data.}
  We define a map from labeled one-hot datasets to a normalized positive semidefinite, unit-trace matrix (density operator), providing a structured matrix representation for supervised categorical learning \cite{NielsenChuang2000,Holevo2001}.

  \item \textbf{Intrinsic low-rank spectral structure and Gram reduction.}
  We prove that the operator rank is bounded by the number of classes and show that its nonzero spectrum can be recovered from a $k\times k$ Gram matrix, yielding an efficient and transparent spectral computation.

  \item \textbf{Hellinger/Bhattacharyya geometric interpretation.}
  We relate the construction to the Bhattacharyya affinity and Hellinger-type geometry on class-conditional categorical profiles, thereby connecting the proposed operator to a classical distributional similarity framework \cite{Bhattacharyya1943,Aherne1998}.

  \item \textbf{Perturbation and subspace stability analysis.}
  We establish spectral-subspace stability via Davis--Kahan type perturbation bounds and derive explicit high-probability operator perturbation estimates under multinomial sampling in blockwise categorical models \cite{DavisKahan1970,YuWangSamworth2015}.

  \item \textbf{Synthetic experiments illustrating spectral behavior.}
  We design controlled synthetic experiments (separation, sparsity, irrelevant variables, and imbalance) to illustrate the empirical behavior of the proposed spectral embedding and its stability trends.
\end{itemize}

\subsection{Paper organization}

Section~\ref{sec:categorical_rep} introduces the categorical survey space and the label encoding.
Section~\ref{sec:density_construction} presents the density-operator construction and its fundamental properties.
Section~\ref{sec:surrogate} develops the spectral embedding and low-dimensional surrogates.
Section~\ref{sec:theory_comparison} provides the theoretical analysis, including the Hellinger/Bhattacharyya interpretation, operator relations, perturbation bounds, and subspace stability estimates.
Section~\ref{sec:kde_classification} introduces latent-space density estimation and the resulting classification rule.
Section~\ref{sec:algorithm} summarizes the computational pipeline and analyzes its complexity.
Section~\ref{sec:synthetic} reports synthetic experiments.
The paper concludes with a discussion of limitations and directions for future work.

\section{Categorical Data Representation}
\label{sec:categorical_rep}

We formalize the categorical input space using a blockwise one-hot encoding that is standard in applications but particularly convenient
for an operator-theoretic treatment. Each observation is described by $q$ categorical variables (survey questions, fields, or events), where
variable $i$ has $m_i\ge 2$ mutually exclusive modalities. The resulting ambient dimension is
\[
d=\sum_{i=1}^q m_i,
\]
and each sample is a sparse vector in $\{0,1\}^d$ with exactly one active coordinate per block. This representation allows us to treat the
dataset as a structured collection of elementary events on a finite product space and, crucially, to aggregate class-conditional frequency
information into matrices whose spectral properties can be analyzed by standard tools in linear algebra.

We also adopt a one-hot encoding for the output labels, so that each labeled sample can be represented as a rank-one tensor
(or outer product) in $\R^{d\times k}$. This tensorial viewpoint is not introduced for computational convenience alone: it provides a natural
bridge between contingency-type tables and positive semidefinite operators constructed from class-conditional counts, which will be the
foundation of the density-operator methodology developed in the next section.
\begin{Notation}[Survey space and dimensions]
We consider a population $\Omega$ and a survey consisting of $q$ categorical questions.
The $i$-th question is encoded by a categorical variable $X_i$ with $m_i\ge 2$ mutually exclusive modalities
$\mathcal O_i=\{O_1^{(i)},\dots,O_{m_i}^{(i)}\}$.
We denote by $d=\sum_{i=1}^q m_i$ the ambient one-hot dimension.
\end{Notation}

Each individual $\omega\in\Omega$ produces an outcome for question $i$ encoded by the one-hot vector
\[
X_i(\omega)=\sum_{k=1}^{m_i}\alpha_k^{(i)}(\omega)\, \mathbf e_k^{(i)}\in\{0,1\}^{m_i},
\]
where $\alpha_k^{(i)}(\omega)=1$ if $\omega$ selects modality $O_k^{(i)}$, and $\alpha_k^{(i)}(\omega)=0$ otherwise.
Here $\mathbf e_k^{(i)}$ denotes the $k$-th canonical basis vector of $\R^{m_i}$.

\begin{Definition}[Categorical state space]
For each question $i$, the set of admissible outcomes is
\[
\mathcal X_{m_i}:=\{\mathbf e_1^{(i)},\dots,\mathbf e_{m_i}^{(i)}\}\subset\{0,1\}^{m_i}.
\]
The full survey outcome space is the Cartesian product
\[
\mathcal X_d := \mathcal X_{m_1}\times\cdots\times \mathcal X_{m_q}\subset\{0,1\}^d.
\]
\end{Definition}

\begin{Example}[Binary case]
If $m_i=2$, then $X_i(\omega)\in\{(1,0)^T,(0,1)^T\}\subset\{0,1\}^2$.
\end{Example}

\begin{Definition}[Survey vector (block-concatenation)]
The survey response of $\omega$ is represented by the concatenation of one-hot blocks,
\[
x(\omega):=\bigl(X_1(\omega),\dots,X_q(\omega)\bigr)\in\{0,1\}^d.
\]
Equivalently, $x(\omega)$ has exactly $q$ entries equal to $1$, one per block.
\end{Definition}


\subsection{Label encoding and tensor representation}
We consider a $k$-class classification problem. The output variable $Y$ takes values in
\[
\mathcal Y_k:=\{\mathbf e_1,\dots,\mathbf e_k\}\subset\{0,1\}^k,\qquad k\ge 2.
\]
Throughout, we focus on the regime $d\gg k$.

\begin{Definition}[Tensor encoding of labeled samples]
Given $x\in\mathcal X_d$ and a label $\mathbf e_y\in\mathcal Y_k$, define
\[
x\otimes \mathbf e_y := x\,\mathbf e_y^{\top}\in\R^{d\times k}.
\]
This matrix has all-zero columns except column $y$, which equals $x$.
The corresponding tensor state space is
\[
\mathcal X_d\otimes\mathcal Y_k :=
\{x\,\mathbf e_y^{\top}: x\in\mathcal X_d,\ \mathbf e_y\in\mathcal Y_k\}.
\]
\end{Definition}

\begin{Remark}[Discretized numerical variables]
Quantitative variables can be incorporated by discretization into bins and subsequent one-hot encoding,
thus fitting within the same survey-space formalism.
\end{Remark}

\begin{Definition}[Training dataset]
We consider a labeled training set
\[
\mathcal D := \{x^{(j)}\otimes \mathbf e_{y^{(j)}} : 1\le j\le n\}\subset \mathcal X_d\otimes\mathcal Y_k.
\]
The goal is to construct an empirical classification map
\[
\ell_{\mathcal D}:\mathcal X_d\to\mathcal Y_k
\]
approximating an ideal (unknown) classifier $\ell$.
\end{Definition}

\section{Constructing a Density Matrix from the Training Set}
\label{sec:density_construction}

\begin{Definition}[Class-conditional frequency vectors]
For each class label $\mathbf e_y\in\mathcal Y_k$, define the frequency vector
\[
\mathbf f_y := \sum_{x\otimes \mathbf e_y\in \mathcal D} x \in \R^d_{\ge 0}.
\]
Equivalently, $(\mathbf f_y)_i$ counts how many times attribute $i$ appears among samples of class $y$.
\end{Definition}

\begin{Definition}[Frequency matrix]
Define the frequency matrix
\[
F := [\mathbf f_1\,\mathbf f_2\,\cdots\,\mathbf f_k]\in\R^{d\times k}.
\]
\end{Definition}

\begin{Definition}[Amplitude lifting]
Define the entrywise square-root lifting
\[
X:=\sqrt{F}\in\R^{d\times k},\qquad X_{i,y}:=\sqrt{F_{i,y}}.
\]
\end{Definition}

\begin{Definition}[Density matrix]
\label{def:density_matrix}
The density matrix induced by $\mathcal D$ is
\[
\rho_{\mathcal D}:=\frac{XX^\top}{\mathrm{tr}(XX^\top)}\in\R^{d\times d}.
\]
\end{Definition}

\begin{Proposition}[Density-matrix axioms]
\label{prop:density_axioms}
$\rho_{\mathcal D}$ is symmetric, positive semidefinite, and satisfies $\mathrm{tr}(\rho_{\mathcal D})=1$.
\end{Proposition}

\begin{proof}
Symmetry is immediate since $XX^\top$ is symmetric.
For any $v\in\R^d$, $v^\top XX^\top v = \|X^\top v\|_2^2\ge 0$, hence $XX^\top\succeq 0$ and $\rho_{\mathcal D}\succeq 0$.
Finally, $\mathrm{tr}(\rho_{\mathcal D})=\mathrm{tr}(XX^\top)/\mathrm{tr}(XX^\top)=1$.
\end{proof}

\begin{Proposition}[Intrinsic rank bound]
\label{prop:rank_bound_mdpi}
$\mathrm{rank}(\rho_{\mathcal D})\le k$.
\end{Proposition}

\begin{proof}
Since $\rho_{\mathcal D}$ is a positive scalar multiple of $XX^\top$, we have
$\mathrm{rank}(\rho_{\mathcal D})=\mathrm{rank}(XX^\top)=\mathrm{rank}(X)$.
Because $X\in\R^{d\times k}$, it holds $\mathrm{rank}(X)\le k$.
\end{proof}

\begin{Remark}[Interpretation as mixed states]
In quantum probability, a symmetric Positive Semidefinite (PSD) unit-trace matrix is a density matrix.
The spectral decomposition of $\rho_{\mathcal D}$ expresses it as a convex combination of rank-one projectors,
providing an operator-theoretic representation of class-conditional attribute structure.
\end{Remark}

\section{Low-Dimensional Surrogates via Spectral Coordinates}
\label{sec:surrogate}

The density operator $\rho_{\mathcal D}$ constructed from class-conditional frequencies is a symmetric positive semidefinite matrix with
unit trace and, by Proposition~\ref{prop:rank_bound_mdpi}, has rank at most the number of classes. Consequently, its spectral decomposition
provides a natural mechanism to reduce the ambient one-hot dimension $d$ to a small number of degrees of freedom while preserving the dominant
invariant subspace encoded by the data. In this section we introduce the associated spectral coordinates and define the resulting surrogate
(training) representation used in subsequent estimation and classification steps.

Let $\rho_{\mathcal D}=\sum_{i=1}^{r_\star}\sigma_i\,\mathbf U_i\mathbf U_i^\top$ be the spectral decomposition of $\rho_{\mathcal D}$,
where $r_\star=\mathrm{rank}(\rho_{\mathcal D})\le k$, $\sigma_i>0$, and $\{\mathbf U_i\}$ is an orthonormal family in $\R^d$.
Each pair $(\sigma_i,\mathbf U_i)$ represents a nonzero eigenvalue and a corresponding eigenvector, and the truncation level $r\le r_\star$
selects the dominant $r$-dimensional invariant subspace used to build the low-dimensional surrogate coordinates.
\begin{Definition}[Normalized survey vector]
For $x\in\mathcal X_d$, define the Euclidean normalization
\[
\widetilde x := \frac{x}{\|x\|_2}=\frac{x}{\sqrt{q}}\in\R^d,
\]
since one-hot survey vectors satisfy $\|x\|_2^2=q$.
\end{Definition}

\begin{Definition}[Spectral coordinates and truncation]
Define the spectral coordinates for a normalized survey vector $\widetilde{x}$ as $\lambda_i(\widetilde x):=\mathbf U_i^\top \widetilde x$.
For $r\le r_\star$, define the latent coordinate map
\[
\pi_r(x):=\bigl(\lambda_1(\widetilde x),\dots,\lambda_r(\widetilde x)\bigr)\in\R^r.
\]
\end{Definition}

\begin{Proposition}[Support constraint]
For any $x\in\mathcal X_d$, the latent coordinate satisfies $\|\pi_r(x)\|_2\le 1$.
\end{Proposition}

\begin{proof}
Since $\{\mathbf U_i\}$ is orthonormal and $\|\widetilde x\|_2=1$,
\[
\|\pi_r(x)\|_2^2=\sum_{i=1}^r (\mathbf U_i^\top \widetilde x)^2
\le \sum_{i=1}^{d} (\mathbf U_i^\top \widetilde x)^2=\| \mathbf U^\top\widetilde x\|_2^2=\|\widetilde x\|_2^2=1.
\]
\end{proof}

\begin{Definition}[Surrogate training set]
The surrogate (embedded) training dataset is
\[
\widetilde{\mathcal D}_r := \{ \pi_r(x^{(j)})\otimes \mathbf e_{y^{(j)}} : 1\le j\le n\}\subset \R^r\otimes \mathcal Y_k.
\]
\end{Definition}

\begin{Remark}[Polar/spherical parametrizations]
Since $\pi_r(x)\in B_1^r(0)$, one may use polar (for $r=2$) or spherical coordinates (for $r\ge 3$)
to obtain interpretable low-dimensional parametrizations.
\end{Remark}


\section{Theoretical Analysis of the Density-Operator Construction}
\label{sec:theory_comparison}

This section develops the linear-algebraic interpretation and stability theory underlying the proposed density-operator construction.
We first introduce a class-normalized operator that separates the geometry of class-conditional categorical profiles from class prevalence,
thereby providing a convenient population viewpoint and enabling sharper perturbation statements.
We then show that the resulting class-to-class Gram matrix admits a Hellinger/Bhattacharyya kernel interpretation, linking the operator to
a well-known distributional geometry and clarifying connections with spectral kernel methods and correspondence-type factorizations.

From a matrix-theoretic perspective, the key question is how the dominant invariant subspace of the empirical operator behaves under
perturbations induced by finite-sample fluctuations in the underlying frequency profiles. We therefore combine a Davis--Kahan
$\sin\Theta$ subspace bound with an explicit high-probability control of the operator perturbation, yielding a quantitative stability
guarantee for the spectral embedding under multinomial sampling.
Finally, we relate the count-based and class-normalized constructions and show that they differ only through class-mass weighting, leading
to explicit norm bounds and conditions under which both operators induce essentially the same embedding.

\begin{Assumption}
    For clarity, we summarize the assumptions used in the theoretical analysis.
    Assumption~\ref{ass:population_and_gap} specifies the blockwise categorical sampling model and introduces the spectral gap
    condition $\lambda_r(\rho_\star)>\lambda_{r+1}(\rho_\star)$ required to control the stability of the dominant invariant subspace via
    Davis--Kahan type bounds. Assumption~\ref{ass:positivity_profiles} is a technical regularity condition ensuring uniform positivity of the
    class profiles; it guarantees that the entrywise square-root map is Lipschitz on a positive interval and allows us to derive explicit
    high-probability bounds for the perturbation norm $\|\widehat\rho_\star-\rho_\star\|_2$. Importantly, both assumptions are only needed to obtain
    quantitative stability estimates and are not required to define the operator construction or the embedding itself.
\end{Assumption}

\begin{Notation}
    To avoid ambiguity, we use $\rho_{\mathrm{CN}}$ for the finite-sample class-normalized operator constructed directly from the count matrix $F$
    (Definition~\ref{def:class_normalized_profiles}), $\rho_\star$ for its population counterpart under the blockwise categorical model, and
    $\widehat\rho_\star$ for the empirical estimator of $\rho_\star$ used in the perturbation analysis. Analogously, $\Psi$ denotes the population
    amplitude matrix and $\Psi_n$ its empirical counterpart in the probabilistic bounds.
\end{Notation}

\subsection{A normalized population model}
\label{subsec:normalized_population_model}

To obtain a mathematically clean comparison framework, we introduce a class-normalized variant of the frequency matrix.
This normalization separates the geometry of class-conditional profiles from class sizes and is particularly convenient for theoretical analysis.

\begin{Definition}[Class-normalized profiles, amplitude matrix, and class-normalized density operator]
\label{def:class_normalized_profiles}
Let $F=(\mathbf f_1\,|\,\mathbf f_2\,|\,\cdots\,|\,\mathbf f_k)\in\R^{d\times k}$ be the class-conditional count matrix, where
$\mathbf f_y\in\R_{\ge 0}^d$ collects the attribute counts for class $y$.
Assume $\mathbf 1^\top \mathbf f_y>0$ for all $1\le y\le k$ and define the class-normalized profiles
\[
\widehat p_y \in \R^d,\qquad
\widehat p_y(i):=\frac{F_{i,y}}{\sum_{j=1}^d F_{j,y}}
\quad (1\le i\le d,\ 1\le y\le k).
\]
Define the corresponding \emph{amplitude profiles} $\widehat\psi_y:=\sqrt{\widehat p_y}\in\R^d$ (entrywise square root) and the
amplitude matrix
\[
\widehat\Psi := [\widehat\psi_1\,\cdots\,\widehat\psi_k]\in\R^{d\times k}.
\]
Finally, define the associated \emph{class-normalized density operator}
\[
\rho_{\mathrm{CN}}
:= \frac{\widehat\Psi\widehat\Psi^\top}{\tr(\widehat\Psi\widehat\Psi^\top)}
\in\R^{d\times d}.
\]
Equivalently, since $\|\widehat\psi_y\|_2=1$ for each $y$, one has
\[
\tr(\widehat\Psi\widehat\Psi^\top)=k,
\qquad\text{hence}\qquad
\rho_{\mathrm{CN}}=\frac{1}{k}\widehat\Psi\widehat\Psi^\top.
\]
In later probabilistic statements (population/empirical perturbation analysis), we reserve the notation
$\widehat\rho_\star$ for the empirical estimator of the population class-normalized operator $\rho_\star$.
\end{Definition}

\paragraph{Relation to the count-based construction and an alternative statistical view}
The operator $\rho_{\mathrm{CN}}$ is a class-normalized analogue of
\[
\rho_{\mathcal D}=XX^\top/\tr(XX^\top),\qquad X=\sqrt{F},
\]
where the square root is taken entrywise.
In practice, both variants are useful: $\rho_{\mathcal D}$ retains class-size effects (through the class masses), while $\rho_{\mathrm{CN}}$
isolates the geometry of class-conditional categorical profiles and leads to sharper theoretical statements.

From a statistical perspective, $\rho_{\mathrm{CN}}$ can also be interpreted as an empirical uncentered second-moment operator of the normalized amplitude profiles.
Indeed, since $\tr(\widehat\Psi\widehat\Psi^\top)=k$, one has
\[
\rho_{\mathrm{CN}}
=
\frac{1}{k}\widehat\Psi\widehat\Psi^\top
=
\frac{1}{k}\sum_{y=1}^k \widehat\psi_y\,\widehat\psi_y^\top,
\]
which is precisely the empirical uncentered second-moment matrix of the vectors $\{\widehat\psi_y\}_{y=1}^k$.
If a centered covariance-type operator is desired, one may replace $\widehat\psi_y$ by $\widehat\psi_y-\bar\psi$, where
\[
\bar\psi:=\frac{1}{k}\sum_{y=1}^k \widehat\psi_y,
\]
and then renormalize accordingly.

\subsection{Hellinger geometry and a kernel-PCA viewpoint}
\label{subsec:hellinger_kernel_pca}

A distinctive aspect of the proposed construction is that the entrywise square-root lifting naturally induces a
Hellinger-type geometry on class-conditional categorical profiles. This subsection makes this connection explicit and
clarifies how the resulting spectral embedding can be interpreted as a kernel-PCA procedure acting on probability
vectors rather than on raw one-hot observations. The underlying similarity is the classical Bhattacharyya coefficient,
a standard affinity measure between probability distributions \cite{Bhattacharyya1943,Aherne1998}.

\begin{Definition}[Hellinger distance and Bhattacharyya affinity]
\label{def:hellinger_bhattacharyya}
Let
\[
\Delta^{d-1}
:= \Big\{ p \in \mathbb{R}^d \;\big|\; p_i \ge 0,\ \sum_{i=1}^d p_i = 1 \Big\}
\]
denote the \((d-1)\)-dimensional probability simplex. Let $p,q\in\Delta^{d-1}$ be discrete probability distributions on $\{1,\dots,d\}$.
The \emph{Bhattacharyya affinity} is
\[
\mathrm{BC}(p,q) := \sum_{i=1}^d \sqrt{p(i)\,q(i)} \in [0,1],
\]
and the \emph{Hellinger distance} is
\[
H(p,q) := \frac{1}{\sqrt{2}}\left\|\sqrt{p}-\sqrt{q}\right\|_2
= \sqrt{1-\mathrm{BC}(p,q)} \in [0,1].
\]
\end{Definition}

\begin{Remark}[Geometric meaning]
The map $\Phi:\Delta^{d-1}\to\R^d$, $\Phi(p)=\sqrt{p}$ (entrywise), embeds the simplex into the positive orthant of the unit sphere,
since $\|\sqrt{p}\|_2^2=\sum_i p(i)=1$.
Thus, $\mathrm{BC}(p,q)=\langle \Phi(p),\Phi(q)\rangle$ is the cosine similarity of the embedded points,
and $H(p,q)$ becomes the (scaled) Euclidean distance on the sphere. This classical viewpoint is widely used for frequency-coded data
and distributional comparisons \cite{Bhattacharyya1943,Aherne1998}.
\end{Remark}

\begin{Proposition}[Hellinger affinity kernel between classes]
\label{char:hellinger_kernel}
Let $\widehat\psi_y=\Phi(\widehat p_y)$ be the amplitude profile of class $y$.
Then the class-class Gram matrix $G\in\R^{k\times k}$ defined by
\[
G := \widehat\Psi^\top\widehat\Psi,\qquad
G_{yy'} = \langle \widehat\psi_y,\widehat\psi_{y'}\rangle
= \sum_{i=1}^d \sqrt{\widehat p_y(i)\,\widehat p_{y'}(i)},
\]
is the empirical Bhattacharyya affinity matrix between the class-conditional categorical distributions
$\widehat p_y$ and $\widehat p_{y'}$.
\end{Proposition}

\begin{Proposition}[Kernel structure and positive semidefiniteness]
\label{prop:hellinger_kernel_psd}
The affinity map
\[
K_{\mathrm{Hel}}(p,q):=\sum_{i=1}^d\sqrt{p(i)\,q(i)} = \langle \Phi(p),\Phi(q)\rangle
\]
defines a positive semidefinite kernel on $\Delta^{d-1}$.
Consequently, $G$ is symmetric positive semidefinite.
\end{Proposition}

\begin{proof}
Since $K_{\mathrm{Hel}}(p,q)=\langle \Phi(p),\Phi(q)\rangle$ with $\Phi(p)=\sqrt{p}\in\R^d$, it is an inner-product kernel.
Therefore, for any coefficients $(a_y)_{y=1}^k$,
\[
\sum_{y,y'=1}^k a_y a_{y'} K_{\mathrm{Hel}}(\widehat p_y,\widehat p_{y'})
=
\left\|\sum_{y=1}^k a_y \sqrt{\widehat p_y}\right\|_2^2 \ge 0,
\]
which proves positive semidefiniteness. Hence $G=\widehat\Psi^\top\widehat\Psi\succeq 0$.
\end{proof}

\paragraph{Relation to kernel-PCA and correspondence analysis}
Kernel principal component analysis (kernel-PCA) constructs low-dimensional coordinates by eigendecomposition of a Gram matrix induced by a
positive semidefinite kernel \cite{Scholkopf1998}. Proposition~\ref{prop:hellinger_kernel_psd} shows that
$G=\widehat\Psi^\top\widehat\Psi$ is exactly the Gram matrix associated with the Hellinger/Bhattacharyya kernel applied to the
class-conditional profiles $(\widehat p_y)_{y=1}^k$. Hence, the spectral directions extracted from $\widehat\Psi\widehat\Psi^\top$
(or equivalently from $G$) admit a kernel-PCA interpretation in the feature space induced by the amplitude map $p\mapsto\sqrt{p}$, where each
class is represented by $\widehat\psi_y=\sqrt{\widehat p_y}$. This also clarifies the connection with correspondence analysis (CA) and
discriminant correspondence analysis (DCA), which derive low-dimensional geometries from contingency tables via SVD under $\chi^2$-type
normalizations \cite{Greenacre1984,DCA}. By contrast, the present construction replaces the $\chi^2$ geometry by the Hellinger geometry induced
by $\sqrt{\cdot}$, which is well suited for comparing discrete distributions via the Bhattacharyya affinity and Hellinger distance
\cite{Bhattacharyya1943,Aherne1998}. Consequently, the resulting spectral representation is naturally compatible with
operator-theoretic stability tools (Section~\ref{subsec:relation_two_operators}) and downstream probabilistic modeling in the latent space.

\subsection{Population operator and subspace stability (Davis--Kahan)}
\label{subsec:davis_kahan_stability}

A central question for any spectral embedding method is stability: how sensitive is the leading eigenspace of the
learned operator to perturbations induced by sampling noise, finite data, or modeling mismatches.
This issue is classical in matrix perturbation theory and is precisely quantified by the Davis--Kahan $\sin\Theta$ theorem,
which bounds the rotation of invariant subspaces of symmetric operators in terms of the perturbation size and a spectral gap.
For completeness, we recall a standard formulation adapted to our setting (see \cite{DavisKahan1970,YuWangSamworth2015}).

\begin{Theorem}[Davis--Kahan $\sin\Theta$ theorem]
\label{thm:davis_kahan_general}
Let $A,\widetilde A\in\R^{d\times d}$ be symmetric matrices and set $E:=\widetilde A-A$.
Fix $r\in\{1,\dots,d-1\}$ and assume a spectral gap
\[
\delta := \lambda_r(A)-\lambda_{r+1}(A) > 0,
\]
where $\lambda_1(A)\ge\lambda_2(A)\ge\cdots\ge\lambda_d(A)$ denote the eigenvalues of $A$.
Let $U_r\in\R^{d\times r}$ (resp.\ $\widetilde U_r\in\R^{d\times r}$) contain as columns the top-$r$ eigenvectors of $A$
(resp.\ $\widetilde A$). Then
\[
\bigl\|\sin\Theta(\widetilde U_r,U_r)\bigr\|_2
\;\le\;
\frac{\|E\|_2}{\delta}
\;=\;
\frac{\|\widetilde A-A\|_2}{\lambda_r(A)-\lambda_{r+1}(A)}.
\]
\end{Theorem}

Theorem~\ref{thm:davis_kahan_general} shows that the leading $r$-dimensional eigenspace is stable whenever the perturbation
$\|\widetilde A-A\|_2$ is small compared to the eigen-gap $\lambda_r(A)-\lambda_{r+1}(A)$.
This provides a sharp and widely used tool to justify spectral embeddings under sampling fluctuations.

\begin{Assumption}[Blockwise categorical model and spectral gap]
\label{ass:population_and_gap}
Assume the training samples $(x^{(j)},y^{(j)})_{j=1}^n$ are i.i.d.
Conditioned on $Y=y$, the survey vector $x\in\{0,1\}^d$ is generated by independent categorical blocks as in the survey model,
and let $p_y\in\Delta^{d-1}$ denote the (population) class-conditional categorical profile.
Define the population amplitudes $\psi_y:=\sqrt{p_y}$ and $\Psi:=[\psi_1\,\cdots\,\psi_k]\in\R^{d\times k}$.
Let the population density operator be
\[
\rho_\star := \frac{\Psi\Psi^\top}{\tr(\Psi\Psi^\top)}.
\]
Fix an embedding dimension $r\le k$ and assume a spectral gap
\[
\delta_r = \lambda_r(\rho_\star) - \lambda_{r+1}(\rho_\star) > 0,
\]
where $\lambda_1(\rho_\star)\ge\lambda_2(\rho_\star)\ge\cdots$ denote the eigenvalues of $\rho_\star$.
\end{Assumption}

\begin{Notation}[Dominant eigenspaces]
\label{not:eigenspaces}
Let $U_{\star,r}\in\R^{d\times r}$ contain the top-$r$ eigenvectors of $\rho_\star$ as columns,
and let $\widehat U_{\star,r}\in\R^{d\times r}$ contain the top-$r$ eigenvectors of $\widehat\rho_\star$.
We denote by $\sin\Theta(\widehat U_{\star,r},U_{\star,r})$ the diagonal matrix of principal-angle sines
between the corresponding $r$-dimensional subspaces.
\end{Notation}

\begin{Theorem}[Spectral subspace stability]
\label{thm:davis_kahan_stability}
Under Assumption~\ref{ass:population_and_gap}, the dominant $r$-dimensional spectral subspace of the empirical operator
$\widehat\rho_\star$ is stable with respect to perturbations of the population operator $\rho_\star$ in the sense that
\[
\bigl\|\sin\Theta(\widehat U_{\star,r},U_{\star,r})\bigr\|_2
\;\le\;
\frac{\|\widehat\rho_\star-\rho_\star\|_2}{\delta_r}.
\]
\end{Theorem}

\begin{proof}
Apply Theorem~\ref{thm:davis_kahan_general} with $A=\rho_\star$ and $\widetilde A=\widehat\rho_\star$.
Then $E=\widehat\rho_\star-\rho_\star$ and the eigen-gap is $\delta_r=\lambda_r(\rho_\star)-\lambda_{r+1}(\rho_\star)$,
which yields the stated bound.
\end{proof}

\subsection{A perturbation bound for \texorpdfstring{$\|\widehat\rho_\star-\rho_\star\|_2$}{||rhohatstar-rhostar||2} under multinomial sampling}
\label{subsec:perturbation_bound}

To complement the Davis--Kahan stability estimate (Theorem~\ref{thm:davis_kahan_stability}), we provide a simple nonasymptotic bound on
$\|\widehat\rho_\star-\rho_\star\|_2$ under the blockwise categorical generator. We use the notation $\widehat\rho_\star$ for the empirical estimator
of the population class-normalized operator $\rho_\star$, to distinguish it from the finite-sample class-normalized operator
$\rho_{\mathrm{CN}}$ introduced in Definition~\ref{def:class_normalized_profiles}.

\begin{Assumption}[Strict positivity of class profiles]
\label{ass:positivity_profiles}
There exists $p_{\min}\in(0,1)$ such that, for all classes $y\in\{1,\dots,k\}$ and all coordinates $i\in\{1,\dots,d\}$,
\[
p_y(i)\ge p_{\min}.
\]
\end{Assumption}

\begin{Lemma}[From factor perturbations to operator perturbations]
\label{lem:rho_lipschitz_in_psi}
Let $\Psi,\Psi_n\in\R^{d\times k}$ and define
\[
\rho_\star:=\frac{1}{k}\Psi\Psi^\top,\qquad \widehat\rho_\star:=\frac{1}{k}\Psi_n\Psi_n^\top.
\]
Then
\[
\|\widehat\rho_\star-\rho_\star\|_2
\;\le\;
\frac{1}{k}\Bigl( \|\Psi_n-\Psi\|_2^2 + 2\|\Psi\|_2\,\|\Psi_n-\Psi\|_2 \Bigr).
\]
\end{Lemma}

\begin{proof}
Write $\Delta:=\Psi_n-\Psi$. It can be verified that $\hat{\rho_\star} = \frac{1}{k}(\Psi+\Delta)(\Psi+\Delta)^\top$. Then:
\[
\widehat\rho_\star-\rho_\star
=
\frac{1}{k}\bigl((\Psi+\Delta)(\Psi+\Delta)^\top-\Psi\Psi^\top\bigr)
=
\frac{1}{k}\bigl(\Psi\Delta^\top+\Delta\Psi^\top+\Delta\Delta^\top\bigr).
\]
Taking spectral norms and using the the Cauchy-Schwarz inequality yields:
\[
\|\widehat\rho_\star-\rho_\star\|_2
\le
\frac{1}{k}\bigl(2\|\Psi\|_2\|\Delta\|_2+\|\Delta\|_2^2\bigr),
\]
which is the claim.
\end{proof}

\begin{Theorem}[High-probability bound for $\|\widehat\rho_\star-\rho_\star\|_2$]
\label{thm:rho_perturbation_multinomial}
Assume the samples are generated as follows: for each class $y$, we observe $n_y\ge 1$ i.i.d.\ labeled samples whose one-hot survey vectors
follow the blockwise categorical model, and let $p_y\in\Delta^{d-1}$ be the theoretical class-conditional profile on the one-hot coordinates.
Let $\widehat p_y$ denote the empirical profile (relative frequencies) in class $y$, and define
\[
\psi_y:=\sqrt{p_y},\qquad \widehat\psi_y:=\sqrt{\widehat p_y}
\quad\text{(entrywise)},
\]
together with
\[
\Psi=[\psi_1,\dots,\psi_k],\qquad \Psi_n=[\widehat\psi_1,\dots,\widehat\psi_k],\qquad
\rho_\star=\frac{1}{k}\Psi\Psi^\top,\qquad \widehat\rho_\star=\frac{1}{k}\Psi_n\Psi_n^\top.
\]
Under Assumption~\ref{ass:positivity_profiles}, let
\[
n_{\min}:=\min_{1\le y\le k} n_y \geq 0.
\]
Fix $\delta\in(0,1)$ and assume, in addition, that
\begin{equation}
\label{eq:positivity_transfer_condition}
\sqrt{\frac{\log(4dk/\delta)}{2n_{\min}}}\le \frac{p_{\min}}{2}.
\end{equation}
Then, with probability at least $1-\delta$, the following hold simultaneously:
\begin{align}
\max_{1\le y\le k}\|\widehat p_y-p_y\|_\infty
&\le
\sqrt{\frac{\log(4dk/\delta)}{2n_{\min}}}, \label{eq:uniform_profile_bound_nmin}\\[0.5em]
\|\widehat\rho_\star-\rho_\star\|_2
&\le
\frac{1}{k}\Bigl(2\|\Psi\|_2\,\varepsilon+\varepsilon^2\Bigr), \label{eq:rho_bound_eps}
\end{align}
where
\begin{equation}
\label{eq:epsilon_def_nmin}
\varepsilon
:=
,
\sqrt{\frac{kd\log(4dk/\delta)}{4p_{\min}n_{\min}}}.
\end{equation}
\end{Theorem}

\begin{proof}
We proceed in three steps.

\medskip
\noindent\emph{Step 1: uniform concentration of empirical class profiles.}
Fix a class $y$ and a coordinate $i$. Conditional on the class count $n_y$, the empirical frequency $\widehat p_y(i)$ is the sample mean of $n_y$
i.i.d.\ Bernoulli random variables with mean $p_y(i)$, with values between $p_{min}$ and 1. Hence Hoeffding's inequality \cite{Hoeffding1963,BoucheronLugosiMassart2013} gives, for every $t>0$,

    
    
    
    
    
    
    
    
    
    

\[
\Pr\!\Bigl(\bigl|\widehat p_y(i)-p_y(i)\bigr|\ge t \Bigr)
\le
2\exp(-2n_y t^2)
\le
2\exp(-2n_{\min} t^2).
\]

A union bound over $y\in\{1,\dots,k\}$ and $i\in\{1,\dots,d\}$ yields
\[
\Pr\!\Bigl(\max_{y,i} |\widehat p_y(i)-p_y(i)|\ge t\Bigr)
\le
2dk\,\exp\Bigl(\frac{-2t^2}{n_{\min}}\Bigr).
\]
Choosing
\[
t_\delta:=\sqrt{\frac{\log(2dk/\delta)}{2n_{\min}}}
\]
makes the right-hand side at most $\delta$. Therefore, with probability at least $1-\delta$,
\[
\max_{1\le y\le k}\|\widehat p_y-p_y\|_\infty \le t_\delta.
\]
which proves  \eqref{eq:uniform_profile_bound_nmin}.

\medskip
\noindent\emph{Step 2: positivity transfer and control of $\|\Psi_n-\Psi\|_2$.}
On the event from Step~1, condition \eqref{eq:positivity_transfer_condition} implies, for $\delta >0$, with probability $1-\delta$, that
\[
\widehat p_y(i)\ge p_y(i)-\sqrt{\frac{\log(4dk/\delta)}{2n_{\min}}}
\ge p_{\min}-\frac{p_{\min}}{2}
=\frac{p_{\min}}{2}
\qquad \forall\, y,i.
\]
Hence both $p_y(i)$ and $\widehat p_y(i)$ lie in $[p_{\min}/2,1]$, and the square root is known to be a Lipschitz function \cite{rudin1976principles} in that interval as $p_{min} \geq 0$. More precisely, it can be shown that: 
\[
\|\widehat\psi_y-\psi_y\|_2
\le
\frac{1}{\sqrt{2p_{\min}}}\,\|\widehat p_y-p_y\|_2
\le
\frac{\sqrt d}{\sqrt{2p_{\min}}}\,\|\widehat p_y-p_y\|_\infty.
\]
Therefore,
\[
\|\Psi_n-\Psi\|_2
\le
\|\Psi_n-\Psi\|_F
=
\Bigl(\sum_{y=1}^k \|\widehat\psi_y-\psi_y\|_2^2\Bigr)^{1/2}
\le
\sqrt{k}\,\frac{\sqrt d}{\sqrt{2p_{\min}}}\,\max_y\|\widehat p_y-p_y\|_\infty.
\]
Combining with \eqref{eq:uniform_profile_bound_nmin} gives, with probability $1 - \delta$, that

\[
\|\Psi_n-\Psi\|_2 \le \varepsilon,
\]
where $\varepsilon$ is defined in \eqref{eq:epsilon_def_nmin}.

\medskip
\noindent\emph{Step 3: operator perturbation bound.}
Applying Lemma~\ref{lem:rho_lipschitz_in_psi} and the estimate $\|\Psi_n-\Psi\|_2\le \varepsilon$, we obtain
\[
\|\widehat\rho_\star-\rho_\star\|_2
\le
\frac{1}{k}\Bigl(2\|\Psi\|_2\,\varepsilon+\varepsilon^2\Bigr),
\]
which is \eqref{eq:rho_bound_eps}.
\end{proof}

\begin{Remark}
    The positivity assumption is used to obtain a uniform Lipschitz constant for the entrywise square-root map $p\mapsto \sqrt p$.
    More precisely, one needs the empirical profiles $\widehat p_y$ to remain in a positive neighborhood of the population profiles.
    Condition \eqref{eq:positivity_transfer_condition} is a sufficient sample-size requirement ensuring this property on the high-probability event.
    In practice, the same effect can also be enforced by standard additive smoothing.

    
    Combining Theorem~\ref{thm:rho_perturbation_multinomial} with the Davis--Kahan subspace perturbation estimate
    (Theorem~\ref{thm:davis_kahan_stability}) yields an explicit (albeit conservative) high-probability control of the principal-angle distance
    between the empirical and population DMM embedding subspaces in terms of $(d,k,n_{\min},p_{\min})$ and the spectral gap.
\end{Remark}

\begin{Corollary}[Stability of latent coordinates up to orthogonal alignment]
\label{cor:latent_coordinate_stability}
Let
\[
\pi_{\star,r}(x):=U_{\star,r}^\top \widetilde x,
\qquad
\widehat\pi_{\star,r}(x):=\widehat U_{\star,r}^\top,
\]
be the population and empirical spectral coordinates of a normalized survey vector,  $\widetilde x:=x/\sqrt q$.
Then there exists an absolute constant $C>0$ such that, for all $x\in\mathcal X_d$,
\[
\inf_{R\in O(r)}
\bigl\|\widehat\pi_{\star,r}(x)-R^\top\pi_{\star,r}(x)\bigr\|_2
\;\le\;
C\,\bigl\|\sin\Theta(\widehat U_{\star,r},U_{\star,r})\bigr\|_2,
\]
where $O(r)$ denotes the orthogonal group in $\R^{r\times r}$.
\end{Corollary}

\begin{proof}
Fix $x\in\mathcal X_d$ and let $R\in O(r)$. Since $\|\widetilde x\|_2=1$,
\[
\|\widehat\pi_{\star,r}(x)-R^\top\pi_{\star,r}(x)\|_2
=
\|\widehat U_{\star,r}^\top\widetilde x-R^\top U_{\star,r}^\top\widetilde x\|_2
=
\|(\widehat U_{\star,r}-U_{\star,r}R)^\top \widetilde x\|_2
\le
\|\widehat U_{\star,r}-U_{\star,r}R\|_2.
\]
Taking the infimum over $R\in O(r)$ gives
\[
\inf_{R\in O(r)}\|\widehat\pi_{\star,r}(x)-R^\top\pi_{\star,r}(x)\|_2
\le
\inf_{R\in O(r)}\|\widehat U_{\star,r}-U_{\star,r}R\|_2.
\]
The latter is controlled (up to an absolute constant) by the principal-angle distance through the Davis–Kahan $\sin \Theta$ theorem:
\[
\inf_{R\in O(r)}\|\widehat U_{\star,r}-U_{\star,r}R\|_2
\le
C\,\|\sin\Theta(\widehat U_{\star,r},U_{\star,r})\|_2,
\]
which proves the claim.
\end{proof}

\begin{Remark}[A principled comparison lever]
Theorem~\ref{thm:davis_kahan_stability} and Corollary~\ref{cor:latent_coordinate_stability} provide a rigorous mechanism
to compare DMM-type spectral embeddings with alternative dimensionality reduction procedures:
in the presence of a spectral gap, the learned embedding is stable under sampling fluctuations and perturbations.
This is particularly relevant in sparse high-cardinality regimes, where distance-based methods (e.g.\ KNN in the ambient one-hot space)
may exhibit large variance due to concentration phenomena, while the DMM embedding is controlled by an operator perturbation bound.
\end{Remark}

\subsection{A minimal worked example}
\label{subsec:toy_example_bhattacharyya}

\begin{Example}[Bhattacharyya affinities in a toy two-class setup]
\label{ex:toy_bhattacharyya}
Consider $k=2$ classes with empirical profiles $\widehat p_1,\widehat p_2\in\Delta^{d-1}$.
Then
\[
G_{12}=\sum_{i=1}^d\sqrt{\widehat p_1(i)\widehat p_2(i)}\in[0,1]
\]
quantifies class similarity: $G_{12}=1$ if and only if $\widehat p_1=\widehat p_2$, while $G_{12}=0$ when the supports are disjoint.
In this setting, the dominant spectral directions of $\rho_{\mathrm{CN}}$ typically emphasize coordinates contributing most to class separation.
\end{Example}

\subsection{Relation between the count-based and class-normalized operators}
\label{subsec:relation_two_operators}

\begin{Notation}[Class masses and normalized profiles]
\label{not:class_masses_profiles}
Let $F\in\R^{d\times k}$ be the class-conditional count matrix with columns $\mathbf f_y$.
Define the class mass (total count) of class $y$ by
\[
s_y := \sum_{i=1}^d F_{i,y} = \mathbf 1^\top \mathbf f_y,
\qquad
S := \sum_{y=1}^k s_y.
\]
Assume $s_y>0$ for all $y$.
Define the class-normalized categorical profiles
\[
p_y := \frac{\mathbf f_y}{s_y}\in\Delta^{d-1},
\qquad
\psi_y := \sqrt{p_y}\in\R^d,
\qquad
\Psi := [\psi_1\,\cdots\,\psi_k]\in\R^{d\times k},
\]
where the square root is taken entrywise.
\end{Notation}

\begin{Proposition}[Exact decomposition and imbalance weights]
\label{prop:exact_relation_rho_hat_rho}
With the notation of Definition~\ref{def:class_normalized_profiles} and Notation~\ref{not:class_masses_profiles}, the two operators admit the explicit forms
\[
\rho_{\mathcal D} \;=\; \frac{1}{S}\,\Psi\,\diag(s_1,\dots,s_k)\,\Psi^\top
\;=\;
\sum_{y=1}^k w_y\,\psi_y\psi_y^\top,
\qquad
w_y := \frac{s_y}{S},
\]
and
\[
\rho_{\mathrm{CN}} \;=\; \frac{1}{k}\,\Psi\Psi^\top
\;=\;
\frac{1}{k}\sum_{y=1}^k \psi_y\psi_y^\top.
\]
In particular, $\rho_{\mathcal D}$ is a convex combination of rank-one matrices $\psi_y\psi_y^\top$ weighted by the
empirical class masses $w_y$, whereas $\rho_{\mathrm{CN}}$ assigns uniform weights $1/k$.
\end{Proposition}

\begin{proof}
Since $\mathbf f_y = s_y p_y$, we have entrywise
\[
\sqrt{\mathbf f_y} = \sqrt{s_y}\,\sqrt{p_y} = \sqrt{s_y}\,\psi_y.
\]
Hence
\[
X = [\sqrt{\mathbf f_1}\,\cdots\,\sqrt{\mathbf f_k}]
   = [\sqrt{s_1}\psi_1\,\cdots\,\sqrt{s_k}\psi_k]
   = \Psi\,\diag(\sqrt{s_1},\dots,\sqrt{s_k}).
\]
Therefore
\[
XX^\top = \Psi\,\diag(s_1,\dots,s_k)\,\Psi^\top.
\]
Moreover, $\|\psi_y\|_2^2=\sum_{i=1}^d p_y(i)=1$, so
\[
\tr(XX^\top)=\sum_{y=1}^k s_y\,\|\psi_y\|_2^2=\sum_{y=1}^k s_y = S.
\]
This yields $\rho_{\mathcal D}=\frac{1}{S}\Psi\diag(s_1,\dots,s_k)\Psi^\top=\sum_y (s_y/S)\psi_y\psi_y^\top$.

For the class-normalized operator, note that $\tr(\Psi\Psi^\top)=\sum_{y=1}^k\|\psi_y\|_2^2=k$, hence
\[
\rho_{\mathrm{CN}} = \frac{1}{k}\Psi\Psi^\top = \frac{1}{k}\sum_y \psi_y\psi_y^\top.
\]
\end{proof}

\begin{Proposition}[Balanced classes imply equivalence]
\label{prop:balanced_equivalence}
If the class masses are balanced, i.e.\ $s_y = S/k$ for all $y$, then
\[
\rho_{\mathcal D} = \rho_{\mathrm{CN}}.
\]
\end{Proposition}

\begin{proof}
If $s_y=S/k$ for all $y$, then $\diag(s_1,\dots,s_k)=(S/k)I_k$ and
\[
\rho_{\mathcal D}=\frac{1}{S}\Psi\left(\frac{S}{k}I_k\right)\Psi^\top=\frac{1}{k}\Psi\Psi^\top=\rho_{\mathrm{CN}}.
\]
\end{proof}

\begin{Proposition}[Operator deviation controlled by class imbalance]
\label{prop:operator_deviation_bound}
Let $w_y:=s_y/S$ and define the imbalance level
\[
\delta_w := \max_{1\le y\le k}\left|w_y-\frac{1}{k}\right|.
\]
Then
\[
\|\rho_{\mathcal D}-\rho_{\mathrm{CN}}\|_2
\;\le\;
\|\Psi\|_2^2\,\delta_w.
\]
In particular, since $\|\Psi\|_2^2=\lambda_{\max}(\Psi^\top\Psi)$, the deviation is controlled by the spectral scale of the
Bhattacharyya/Hellinger Gram matrix and by the class-mass imbalance.
\end{Proposition}

\begin{proof}
Using Proposition~\ref{prop:exact_relation_rho_hat_rho}, we write
\[
\rho_{\mathcal D}-\rho_{\mathrm{CN}}
=
\Psi\Bigl(\diag(w_1,\dots,w_k)-\frac{1}{k}I_k\Bigr)\Psi^\top.
\]
Taking spectral norms and using submultiplicativity,
\[
\|\rho_{\mathcal D}-\rho_{\mathrm{CN}}\|_2
\le
\|\Psi\|_2^2\,
\left\|\diag(w_1,\dots,w_k)-\frac{1}{k}I_k\right\|_2.
\]
Since the last factor is diagonal, its spectral norm equals the maximum absolute diagonal entry, namely $\delta_w$.
\end{proof}

\begin{Remark}[Practical implication]
Proposition~\ref{prop:exact_relation_rho_hat_rho} shows that $\rho_{\mathcal D}$ and $\rho_{\mathrm{CN}}$ differ only by how they weight classes.
Thus, $\rho_{\mathrm{CN}}$ is preferable when one aims to remove the influence of class prevalence, while $\rho_{\mathcal D}$
can be desirable when prevalence is informative or when the goal is to fit the empirical distribution of the full dataset.
Proposition~\ref{prop:operator_deviation_bound} further shows that when the class imbalance $\delta_w$ is small,
both operators are close and therefore induce similar embeddings.
\end{Remark}

\begin{Corollary}[Embedding equivalence under mild imbalance]
\label{cor:rho_hat_rho_embedding_equivalence}
Let $\rho_{\mathcal D}$ and $\rho_{\mathrm{CN}}$ be the count-based and class-normalized operators defined above.
Fix $r\le k$ and assume that $\rho_{\mathrm{CN}}$ satisfies the spectral gap condition
\[
\lambda_r(\rho_{\mathrm{CN}})>\lambda_{r+1}(\rho_{\mathrm{CN}}).
\]
Let $U^{(\mathrm{CN})}_r\in\R^{d\times r}$ be the matrix of the top-$r$ eigenvectors of $\rho_{\mathrm{CN}}$, and let
$U^{(\mathcal D)}_r\in\R^{d\times r}$ be the matrix of the top-$r$ eigenvectors of $\rho_{\mathcal D}$.
Then
\[
\bigl\|\sin\Theta\bigl(U^{(\mathcal D)}_r,U^{(\mathrm{CN})}_r\bigr)\bigr\|_2
\;\le\;
\frac{\|\rho_{\mathcal D}-\rho_{\mathrm{CN}}\|_2}{\lambda_r(\rho_{\mathrm{CN}})-\lambda_{r+1}(\rho_{\mathrm{CN}})}
\;\le\;
\frac{\|\Psi\|_2^2\,\delta_w}{\lambda_r(\rho_{\mathrm{CN}})-\lambda_{r+1}(\rho_{\mathrm{CN}})},
\]
where $\Psi=[\psi_1\,\cdots\,\psi_k]$ and $\delta_w=\max_y\left|w_y-\frac{1}{k}\right|$ as in
Proposition~\ref{prop:operator_deviation_bound}. In particular, if the imbalance level $\delta_w$ is small, then the dominant
$r$-dimensional embeddings induced by $\rho_{\mathcal D}$ and $\rho_{\mathrm{CN}}$ are close (up to a rotation) in the sense of principal angles.
\end{Corollary}

\begin{proof}
Apply Theorem~\ref{thm:davis_kahan_general} with
\[
A=\rho_{\mathrm{CN}},\qquad \widetilde A=\rho_{\mathcal D}.
\]
This yields
\[
\bigl\|\sin\Theta\bigl(U^{(\mathcal D)}_r,U^{(\mathrm{CN})}_r\bigr)\bigr\|_2
\le
\frac{\|\rho_{\mathcal D}-\rho_{\mathrm{CN}}\|_2}{\lambda_r(\rho_{\mathrm{CN}})-\lambda_{r+1}(\rho_{\mathrm{CN}})}.
\]
The second inequality follows from Proposition~\ref{prop:operator_deviation_bound}.
\end{proof}
\section{Latent-Space Classification via Class-Conditional Density Estimation}
\label{sec:kde_classification}

The spectral map $\pi_r:\mathcal X_d\to \R^r$ transforms sparse one-hot survey vectors into a compact latent representation.
In this section, we perform classification in the latent space by estimating, for each class, a class-conditional density and
then applying a maximum a posteriori (MAP) rule. We adopt kernel density estimation (KDE), a classical nonparametric methodology
introduced by Rosenblatt and Parzen \cite{Rosenblatt1956,Parzen1962} and extensively developed in the statistics literature
\cite{Silverman1986,Scott2015,WandJones1995}.

\subsection{Class-conditional latent clouds}

\begin{Definition}[Class-conditional latent clouds]
\label{def:latent_clouds}
For each class $y\in\{1,\dots,k\}$, define the latent sample cloud
\[
\Pi_y:=\{\pi_r(x^{(j)}):y^{(j)}=y\}\subset \R^r,
\qquad
n_y := |\Pi_y|.
\]
\end{Definition}

\begin{Remark}[Support constraint in the latent space]
By construction, $\pi_r(x)\in B_1^r(0)$ for all $x\in\mathcal X_d$ (Property in Section~\ref{sec:surrogate}).
This suggests two practically relevant kernel choices:
(i) smooth kernels on $\R^r$ (e.g.\ Gaussian), which ignore the boundary but are robust;
(ii) compactly supported kernels (e.g.\ Epanechnikov-type), which can enforce the latent support constraint.
\end{Remark}

\subsection{Kernel density estimation in $\R^r$}

In the following, we will use the notation proposed by \cite{Scott2015}.

\begin{Definition}[Kernel density estimator in $\R^r$]
\label{def:kde_classconditional}
Let $K:\R^r\to[0,\infty)$ satisfy $\int_{\R^r}K(u)\,du=1$ and let $h>0$ be a bandwidth.
For each class $y$, define the KDE
\[
\widehat f_h(z\mid y)
=
\frac{1}{n_y\,h^r}
\sum_{w\in\Pi_y}
K\!\left(\frac{z-w}{h}\right),
\qquad z\in\R^r.
\]
\end{Definition}

\begin{Proposition}[Basic properties of the estimator]
\label{prop:kde_properties}
Assume $K$ is integrable with unit mass. Then for each $y$,
\[
\widehat f_h(\,\cdot\mid y)\ge 0,
\qquad
\int_{\R^r}\widehat f_h(z\mid y)\,dz = 1.
\]
\end{Proposition}

\begin{proof}
Non-negativity is immediate from $K\ge 0$.
For the normalization, by a change of variables $u=(z-w)/h$,
\[
\int_{\R^r}\widehat f_h(z\mid y)\,dz
=
\frac{1}{n_y\,h^r}\sum_{w\in\Pi_y}\int_{\R^r}K\!\left(\frac{z-w}{h}\right)\,dz
=
\frac{1}{n_y}\sum_{w\in\Pi_y}\int_{\R^r}K(u)\,du
=1.
\]
\end{proof}

\begin{Remark}[Bandwidth selection]
The bandwidth $h$ controls the bias--variance trade-off of KDE and is crucial for classification accuracy.
In practice, one may choose $h$ via rules of thumb, plug-in estimators, or cross-validation procedures; see
\cite{Silverman1986,Scott2015,WandJones1995} for standard guidance.
\end{Remark}

\subsection{Posterior scores and decision rules}

\begin{Definition}[Empirical class priors]
\label{def:empirical_priors}
Define the empirical class prior
\[
\widehat \pi_y := \frac{n_y}{n},
\qquad y\in\{1,\dots,k\}.
\]
More generally, one may prescribe priors $\pi_y\in(0,1)$ with $\sum_y \pi_y=1$ to reflect application-specific costs or imbalance.
\end{Definition}

\begin{Definition}[MAP classifier in latent space]
\label{def:map_classifier}
Given class priors $(\pi_y)_{y=1}^k$, the induced classifier in latent space is
\[
\ell_{\widetilde{\mathcal D}_r}(z)
=
\mathbf e_{\widehat y},
\qquad
\widehat y \in \arg\max_{1\le y\le k}\ \pi_y\,\widehat f_h(z\mid y).
\]
The final classifier in the original survey space is
\[
\ell_{\mathcal D}(x):=\ell_{\widetilde{\mathcal D}_r}(\pi_r(x)).
\]
\end{Definition}

\paragraph{Decision rules and numerical implementation}
If all classes are assigned equal priors (or if priors are ignored), then the MAP rule reduces to the maximum-likelihood rule
\[
\widehat y \in \arg\max_{y}\ \widehat f_h(z\mid y).
\]
When the sampling distribution is imbalanced, incorporating priors $\pi_y$ (e.g.\ empirical priors $\widehat\pi_y$) aligns the decision rule
with Bayes risk under the chosen prior; however, as illustrated in Experiment~S4, this may reduce minority-sensitive metrics such as balanced
accuracy and macro-F1 when the prior is highly skewed.

In numerical implementations, it is standard to compute the decision rule using log-scores,
\[
\widehat y \in \arg\max_{y}\ \Bigl(\log \pi_y + \log \widehat f_h(z\mid y)\Bigr),
\]
which improves robustness when $\widehat f_h(z\mid y)$ becomes very small in moderate or high dimensions.

\subsection{Consistency viewpoint (optional theoretical context)}

\paragraph{Bayes rule and plug-in classification}
Let $f(z\mid y)$ denote the true class-conditional density in the latent space and let $\pi_y$ be the class prior.
The Bayes decision rule is the MAP classifier
\[
y^\star(z)\in \arg\max_{y}\ \pi_y\, f(z\mid y),
\]
which minimizes the misclassification risk among all measurable decision rules. A standard approach is to construct a plug-in classifier by
replacing $f(z\mid y)$ with a nonparametric estimate. In our setting, we use kernel density estimation (KDE), following the classical
Parzen--Rosenblatt framework \cite{Rosenblatt1956,Parzen1962,Silverman1986}, leading to an explicit likelihood (or posterior) score in the
reduced coordinates.

\paragraph{Embedding dimension as a spectral truncation parameter}
The truncation level $r$ is a purely spectral parameter: it determines the dimension of the invariant subspace used for the embedding and hence
controls the number of retained degrees of freedom. From the standpoint of density estimation, small $r$ improves concentration and numerical
stability of KDE, whereas overly large $r$ may degrade performance due to the curse of dimensionality. This motivates selecting $r$ using
spectral decay criteria (e.g.\ eigenvalue gaps of the density operator) or validation-based tuning, depending on whether the emphasis is on
operator-theoretic structure or predictive performance.

The previous definitions specify the latent-space classifier used throughout the paper.
We next summarize the end-to-end computational pipeline and derive its complexity, including an efficient low-rank implementation based on the class Gram matrix.


\section{Algorithmic Realization and Complexity Analysis}
\label{sec:algorithm}

Having defined the spectral embedding and the latent-space KDE/MAP classifier, we now summarize the full computational pipeline and its matrix-based implementation. The key observation is that the density operator is positive semidefinite and intrinsically low-rank
(Proposition~\ref{prop:rank_bound_mdpi}), so its nonzero spectral information can be recovered from a $k\times k$ Gram matrix rather than from
a full $d\times d$ eigendecomposition. This section therefore presents (i) the end-to-end procedure for constructing the operator, extracting
a truncated invariant subspace, and performing latent-space classification, and (ii) a complexity analysis that highlights the regime
$k\ll d$ typical of high-cardinality one-hot encodings.

\paragraph{Notation convention}
In this section, the algorithm is written for the \emph{count-based} operator $\rho_{\mathcal D}$.
When discussing the class-normalized alternative introduced in Section~\ref{sec:theory_comparison}, we use the notation
$\rho_{\mathrm{CN}}$ (not $\widehat\rho$) to avoid collision with the empirical population estimator notation used in the perturbation analysis.

\subsection{Algorithm summary}
\label{subsec:algorithm_summary}

Algorithm~\ref{alg:dmm} summarizes the full pipeline.

\begin{algorithm}[H]
\caption{Density-Matrix Spectral Embedding + KDE Classification (count-based operator)}
\label{alg:dmm}
\begin{algorithmic}[1]
\Require Labeled dataset $\mathcal D=\{(x^{(j)},y^{(j)})\}_{j=1}^n$, number of classes $k$, truncation $r\le k$, kernel $K$, bandwidth $h$
\Ensure Classifier $\ell_{\mathcal D}(\cdot)$
\State Encode each sample as one-hot survey vector $x^{(j)}\in\{0,1\}^d$
\State Build class counts $\mathbf f_y\gets \sum_{j:\,y^{(j)}=y}x^{(j)}$ for $y=1,\dots,k$
\State Form $F\gets [\mathbf f_1|\cdots|\mathbf f_k]\in\R^{d\times k}$
\State Compute amplitudes $X\gets \sqrt{F}$ (entrywise)
\State Compute count-based density operator $\rho_{\mathcal D}\gets XX^\top/\tr(XX^\top)$
\State Compute top-$r$ eigenvectors $U_r\in\R^{d\times r}$ of $\rho_{\mathcal D}$ (via the Gram route; see Section~\ref{subsec:gram_equivalence})
\State Define embedding $\pi_r(x)\gets U_r^\top \bigl(x/\sqrt{q}\bigr)$
\State For each class $y$, fit KDE $\widehat f_h(\cdot\mid y)$ on $\Pi_y=\{\pi_r(x^{(j)}):y^{(j)}=y\}$
\State Predict $\ell_{\mathcal D}(x)\gets \arg\max_{y}\ \widehat f_h\bigl(\pi_r(x)\mid y\bigr)$
\end{algorithmic}
\end{algorithm}


\subsection{Nonzero spectrum and eigenvectors via the class Gram matrix}
\label{subsec:gram_equivalence}

The operator $\rho_{\mathcal D}=XX^\top/\tr(XX^\top)$ is a $d\times d$ PSD matrix with $\rank(\rho_{\mathcal D})\le k$
(Proposition~\ref{prop:rank_bound_mdpi}). This low-rank structure implies that all nonzero spectral information of $\rho_{\mathcal D}$
is already contained in the much smaller $k\times k$ Gram matrix $G:=X^\top X$. The next result formalizes this equivalence and
provides an explicit construction of eigenvectors of $XX^\top$ from those of $G$.

The following proposition is classical (see \cite{Strang2006}):

\begin{Proposition}[Nonzero spectra of $XX^\top$ and $X^\top X$]
    \label{prop:gram_spectrum_equivalence}
    Let $X\in\R^{d\times k}$ and set
    \[
    A:=XX^\top\in\R^{d\times d},\qquad G:=X^\top X\in\R^{k\times k}.
    \]
    Then the following statements hold.
    \begin{enumerate}
    \item The matrices $A$ and $G$ are symmetric positive semidefinite and have the same multiset of nonzero eigenvalues (counting algebraic multiplicities). In particular,
    \[
    \rank(A)=\rank(G)=\rank(X)\le k.
    \]
    \item If $v\in\R^k$ satisfies $Gv=\lambda v$ for some $\lambda>0$, then $u:=Xv\in\R^d$ is nonzero and satisfies $Au=\lambda u$.
    Moreover, $\|u\|_2^2=\lambda\|v\|_2^2$.
    \item If $u\in\R^d$ satisfies $Au=\lambda u$ for some $\lambda>0$, then $v:=X^\top u\in\R^k$ is nonzero and satisfies $Gv=\lambda v$.
    \end{enumerate}
\end{Proposition}

\begin{Corollary}[Eigenpairs of $\rho_{\mathcal D}$ from the Gram matrix]
\label{cor:rho_from_gram}
Let $\rho_{\mathcal D}=XX^\top/\tr(XX^\top)$ and $G=X^\top X$.
If $(\lambda_i,v_i)$ are eigenpairs of $G$ with $\lambda_i>0$ and $\|v_i\|_2=1$, then the eigenpairs of $\rho_\mathcal{D}, (\sigma_i, u_i)$, are given by:
\[
\sigma_i=\frac{\lambda_i}{\tr(XX^\top)}
\quad\text{and}\quad
u_i=\frac{1}{\sqrt{\lambda_i}}Xv_i
\]
and satisfy $\|u_i\|_2=1$.
\end{Corollary}

\begin{proof}
By Proposition~\ref{prop:gram_spectrum_equivalence}(2), $XX^\top(Xv_i)=\lambda_i(Xv_i)$.
Dividing by $\tr(XX^\top)$ gives an eigenpair for $\rho_{\mathcal D}$, and the normalization
$\|u_i\|_2=1$ follows from Proposition~\ref{prop:gram_spectrum_equivalence}(2).
\end{proof}

\subsection{Computational complexity}
\label{subsec:complexity}

Let $n$ be the number of samples, $d$ the one-hot dimension, $q$ the number of categorical questions (blocks),
and $k$ the number of classes. We also denote by $r\le k$ the latent embedding dimension retained in the spectral truncation.

\paragraph{Step 1: Building class-conditional counts}
Constructing the frequency matrix $F=[\mathbf f_1|\cdots|\mathbf f_k]\in\R^{d\times k}$ requires a single pass over the dataset.
In the worst case (dense representation) this costs $O(nd)$.
However, the one-hot survey structure implies that each sample has exactly $q$ nonzero entries, hence the updates can be performed
in $O(q)$ time per sample. Therefore, exploiting sparsity yields an overall cost
\[
O(nq),
\]
with memory $O(dk)$ for storing $F$ (or less with streamed / sparse accumulation strategies).

\paragraph{Step 2: Amplitude lifting}
The amplitude matrix $X=\sqrt{F}\in\R^{d\times k}$ is computed entrywise in $O(dk)$ operations.
When $F$ is sparse (typical in high-cardinality settings), this step can be reduced to $O(\mathrm{nnz}(F))$,
where $\mathrm{nnz}(F)$ denotes the number of nonzero entries.

\paragraph{Step 3: Spectral decomposition via a low-rank route}
Computing $\rho_{\mathcal D}=XX^\top/\tr(XX^\top)$ explicitly costs $O(d^2k)$ to form $XX^\top$ and requires $O(d^2)$ memory,
which is prohibitive for large $d$.
Instead, we exploit the intrinsic bound $\rank(\rho_{\mathcal D})\le k$ (Proposition~\ref{prop:rank_bound_mdpi}) and compute
the spectral information from the smaller Gram matrix
\[
G := X^\top X \in \R^{k\times k}.
\]
Forming $G$ costs $O(dk^2)$ (or $O(\mathrm{nnz}(X)\,k)$ in sparse arithmetic).
If $G=V\Lambda V^\top$ is an eigendecomposition with positive eigenvalues $\lambda_1,\dots,\lambda_s$, then the corresponding
nonzero eigenvectors of $XX^\top$ are obtained as
\[
U_s = X V_s \Lambda_s^{-1/2},
\qquad
XX^\top = U_s \Lambda_s U_s^\top,
\]
where $V_s$ collects the eigenvectors associated with the positive spectrum and $\Lambda_s=\diag(\lambda_1,\dots,\lambda_s)$.
The eigendecomposition of $G$ costs $O(k^3)$, so the total complexity of the spectral stage is
\[
O(dk^2+k^3),
\]
with memory $O(dk+k^2)$, which is feasible when $k\ll d$.

\paragraph{Step 4: Computing embeddings}
Given the top-$r$ eigenvectors $U_r\in\R^{d\times r}$, embedding a new observation $x\in\mathcal X_d$ requires
\[
\pi_r(x)=U_r^\top \widetilde x,
\qquad \widetilde x=x/\sqrt{q}.
\]
In dense arithmetic this costs $O(dr)$ per query.
Using the one-hot block sparsity (exactly $q$ ones), the computation reduces to
\[
O(qr),
\]
since $U_r^\top x$ only needs the $q$ active coordinates of $x$.

\paragraph{Step 5: KDE training and prediction}
Constructing the latent clouds $\Pi_y$ requires embedding all samples, costing $O(nqr)$ with sparsity (or $O(ndr)$ in dense arithmetic).
KDE model storage is $O(nr)$ in total.
For a naive KDE evaluation, the cost per query is
\[
O\!\left(\sum_{y=1}^k n_y r\right)=O(nr),
\]
assuming one class-conditional density is evaluated per class and then compared.
Standard accelerations (tree-based methods, fast Gauss transforms, or low-rank kernel approximations) can reduce this
when $n$ is large, but in the typical regime of moderate $r$ and moderate $n$ the naive evaluation is already practical.

\paragraph{Overall cost summary}
In the sparse regime (one-hot blocks), the dominant costs are
\[
\boxed{
\text{Training: } O(nq) \;+\; O(dk^2+k^3) \;+\; O(nqr),
\qquad
\text{Query: } O(qr) \;+\; O(nr).
}
\]
The spectral step $O(dk^2+k^3)$ is typically the bottleneck, but it scales with $k$ rather than with $d^2$ and thus remains
efficient for $k\ll d$.

\begin{Remark}[Scalability regime]
The method is particularly attractive when $k$ is moderate and $d$ is very large, a common situation in high-cardinality categorical learning.
The low-rank route via $G=X^\top X$ avoids forming any $d\times d$ matrices and yields an embedding stage that can be evaluated
in $O(qr)$ time per query using sparsity, making the full pipeline suitable for large-scale one-hot survey representations.
\end{Remark}

\begin{Remark}[Switching to the class-normalized operator]
The class-normalized operator $\rho_{\mathrm{CN}}$ from Section~\ref{sec:theory_comparison} can be computed with the same asymptotic complexity profile,
since it only modifies the column scalings / normalizations before forming the Gram matrix. Therefore, switching between
$\rho_{\mathcal D}$ and $\rho_{\mathrm{CN}}$ does not alter the asymptotic computational cost of the spectral stage.
\end{Remark}

\section{Synthetic Experiments}
\label{sec:synthetic}
This section evaluates the proposed methodology on synthetic datasets designed to test controlled phenomena relevant to categorical learning.
The objective is not to optimize benchmark performance on a specific application domain, but to characterize how the model behaves under increasing sparsity, high cardinality, irrelevant variables, and imbalance.

\subsection{Synthetic generator: categorical blocks}
We consider $q$ categorical variables with modality sizes $m_1,\dots,m_q$ and total dimension $d=\sum_i m_i$.
A sample is generated by selecting one modality in each block.
For each class $c\in\{1,\dots,k\}$ and each variable $i$, we define a categorical distribution
$p_{i,c}\in\Delta^{m_i-1}$, where $\Delta^{m_i-1}$ denotes the simplex.
A sample in class $c$ is generated by independently sampling modalities according to $p_{i,c}$ and concatenating the corresponding one-hot vectors.

We control class separability by adjusting the divergence between $(p_{i,c})_c$ for selected informative variables.

\subsection{Experiment S1: controlled separation}
\label{subsec:synthetic_s1}

This experiment evaluates the behavior of the proposed embedding--classification pipeline under a controlled class-separation
mechanism. Recall that for each categorical block $i\in\{1,\dots,q\}$ and class $c\in\{1,\dots,k\}$ we define a categorical law
$p_{i,c}\in\Delta^{m_i-1}$. A subset of informative variables $I\subset\{1,\dots,q\}$ is selected, and for $i\in I$ the class-conditional
profiles $(p_{i,c})_{c=1}^k$ are made increasingly distinct by a separation parameter $\delta\in[0,1]$.
For non-informative blocks $i\notin I$, we set $p_{i,c}$ identical across classes, so that such variables carry no discriminative signal.

The quantitative results are summarized in Table~\ref{tab:s1_results_compact}, which reports accuracy and macro-F1 for all methods
as a function of the separation parameter $\delta$. As discussed below, the table clearly illustrates the transition from chance-level
performance at $\delta=0$ to near-perfect classification for $\delta\ge 0.6$ across all competitive baselines, while highlighting the
advantage of DMM+KDE over PCA+KNN in the moderate-separation regime (notably at $\delta=0.2$).

\begin{table}[H]
\centering
\caption{Experiment S1 (controlled separation): compact summary of accuracy and macro-F1 as a function of $\delta$.
Each entry is reported as \texttt{accuracy / macro-F1}.}
\label{tab:s1_results_compact}
\setlength{\tabcolsep}{6pt}
\renewcommand{\arraystretch}{1.12}
\begin{tabular}{c c c c c}
\toprule
$\delta$ & DMM+KDE & PCA+KNN & LDA & RF \\
\midrule
0.0 & 0.344 / 0.311 & 0.345 / 0.342 & 0.335 / 0.333 & 0.328 / 0.326 \\
0.2 & 0.586 / 0.585 & 0.534 / 0.532 & 0.578 / 0.578 & 0.586 / 0.586 \\
0.4 & 0.790 / 0.790 & 0.777 / 0.776 & 0.790 / 0.790 & 0.791 / 0.791 \\
0.6 & 0.919 / 0.920 & 0.919 / 0.919 & 0.920 / 0.921 & 0.920 / 0.920 \\
0.8 & 0.983 / 0.983 & 0.981 / 0.981 & 0.982 / 0.983 & 0.983 / 0.983 \\
1.0 & 1.0 / 1.0 & 1.0 / 1.0 & 0.339 / 0.337 & 1.0 / 1.0 \\
\bottomrule
\end{tabular}
\end{table}

Table~\ref{tab:s1_results_compact} displays accuracy and macro-F1 for DMM+KDE, PCA+KNN, LDA, and a random forest (RF)
baseline as $\delta$ increases.
At $\delta=0$, all class-conditional distributions coincide on the informative blocks, hence $X$ is statistically independent of $Y$ and
no classifier can perform better than chance. This is confirmed by the results: all methods yield accuracy close to $1/k$ (here approximately
$0.33$), with macro-F1 essentially matching accuracy, as expected in a balanced multi-class setting.

As $\delta$ increases, the informative blocks become progressively more class-specific, and all methods exhibit a monotone improvement
in both accuracy and macro-F1. In the moderate-separation regime ($\delta=0.2$ and $\delta=0.4$), DMM+KDE is consistently competitive,
outperforming PCA+KNN at $\delta=0.2$ and matching the strongest baselines (RF and LDA) by $\delta=0.4$.
In the high-separation regime ($\delta\ge 0.6$), all methods except LDA approach the ceiling, reaching near-perfect classification
($\approx 0.92$ at $\delta=0.6$ and $\approx 0.98$ at $\delta=0.8$), and eventually achieving perfect accuracy at $\delta=1.0$
for DMM+KDE, PCA+KNN, and RF.

The main advantage of DMM+KDE appears in the intermediate regime where separability exists but is not overwhelming.
In this setting, the density-matrix operator aggregates class-conditional frequency structure through the amplitude lifting $\sqrt{\cdot}$,
which induces the Hellinger/Bhattacharyya geometry discussed in Section~\ref{subsec:hellinger_kernel_pca}.
The resulting spectral coordinates amplify stable, class-consistent categorical patterns while discarding nuisance degrees of freedom,
which explains the observed gain relative to PCA+KNN at $\delta=0.2$.
As separability becomes strong, the classification problem becomes essentially trivial and the performance of all flexible baselines saturates.

While LDA tracks the general improvement trend up to $\delta=0.8$, it collapses to chance-level performance at $\delta=1.0$ in this synthetic
setting. This phenomenon is consistent with the fact that, under extreme categorical separation in high-dimensional one-hot representations,
the within-class covariance estimate may become ill-conditioned or effectively degenerate, which can destabilize Gaussian linear-discriminant
procedures unless additional regularization is introduced (e.g., shrinkage estimators).
In contrast, DMM+KDE and RF remain stable in this regime, as they do not rely on inverting a full covariance operator in the ambient space.

Overall, Experiment~S1 confirms the expected behavior of the proposed method:
(i) it reduces to chance performance when no signal is present ($\delta=0$);
(ii) it improves smoothly with increasing separation and remains competitive with strong baselines;
and (iii) it achieves perfect classification in the fully separable regime ($\delta=1$), while maintaining stable macro-F1 consistent with
balanced class performance.

\subsection{Experiment S2: high cardinality and sparsity}
\label{subsec:synthetic_s2}

Experiment~S2 evaluates the impact of increasing categorical cardinality on classification performance while keeping the sample size $n$
fixed. Increasing the modality sizes $(m_i)$ enlarges the ambient one-hot dimension $d=\sum_{i=1}^q m_i$, producing progressively sparser
representations (each sample still contains exactly $q$ active entries).
Table~\ref{tab:s2_results_compact} reports accuracy and macro-F1 as a function of the scaling factor and the resulting dimension $d$.

\begin{table}[H]
\centering
\caption{Experiment S2 (high cardinality and sparsity): compact summary as a function of the scaling factor and the resulting ambient
dimension $d$. Each entry is reported as \texttt{accuracy / macro-F1}.}
\label{tab:s2_results_compact}
\setlength{\tabcolsep}{6pt}
\renewcommand{\arraystretch}{1.12}
\begin{tabular}{c c c c c}
\toprule
Scale & $d$ & DMM+KDE & PCA+KNN & RF \\
\midrule
1  & 75  & 0.979167 / 0.979039 & 0.978000 / 0.977877 & 0.978000 / 0.977933 \\
2  & 100 & 0.977333 / 0.977202 & 0.978667 / 0.978554 & 0.976833 / 0.976741 \\
4  & 150 & 0.977833 / 0.977701 & 0.978833 / 0.978718 & 0.976500 / 0.976411 \\
8  & 250 & 0.977500 / 0.977350 & 0.978000 / 0.977876 & 0.976167 / 0.976069 \\
\bottomrule
\end{tabular}
\end{table}

Across the considered range ($d$ from $75$ to $250$), all methods remain in a high-performance regime, with accuracy and macro-F1
staying close to $0.98$. Notably, DMM+KDE exhibits very small variation as $d$ increases, confirming that the spectral embedding is
robust under growing ambient dimensionality. This behavior is consistent with the intrinsic low-rank structure of the density operator
(Proposition~\ref{prop:rank_bound_mdpi}), which limits the effective degrees of freedom of the embedding despite the expansion of the
one-hot space.

PCA+KNN remains similarly stable in this regime, suggesting that the induced low-dimensional representation retains sufficient
class-discriminative structure even as sparsity increases. In contrast, RF displays a mild monotone degradation as $d$ grows, which is
consistent with the higher variance and feature-selection burden of tree-based learners in very sparse, high-dimensional one-hot settings.
Finally, macro-F1 closely tracks accuracy in all configurations, indicating balanced performance across classes throughout the sweep.

\subsection{Experiment S3: Irrelevant variables and noise robustness}
\label{subsec:exp_s3}

In Experiment~S3 we evaluate robustness to nuisance dimensions by augmenting the survey space with
$q_{\mathrm{noise}}$ additional categorical blocks whose distributions are identical across classes.
Equivalently, these extra variables carry no label information and act as purely irrelevant covariates.
We parameterize the noise level by $\alpha:=q_{\mathrm{noise}}/q$, where $q$ is the number of original survey blocks.
Table~\ref{tab:s3_results_compact} reports accuracy and macro-F1 for DMM+KDE and two representative baselines.

Overall, performance remains essentially unchanged across the tested noise regimes.
In particular, DMM+KDE preserves high accuracy and macro-F1 (around $0.98$) even when the number of noisy blocks is comparable
to or larger than the number of original blocks (up to $\alpha=2.0$). A similar behavior is observed for PCA+KNN and RF, indicating
that, under this controlled generator, the discriminative signal concentrated on the informative blocks is not substantially diluted
by the addition of class-independent categorical noise.

\begin{table}[H]
\centering
\caption{Experiment S3 (irrelevant variables and noise robustness): accuracy and macro-F1 for increasing numbers of irrelevant categorical blocks.
Each entry is reported as \texttt{accuracy / macro-F1}.}
\label{tab:s3_results_compact}
\setlength{\tabcolsep}{6pt}
\renewcommand{\arraystretch}{1.12}
\begin{tabular}{c c c c c}
\toprule
$\alpha$ & $q_{\mathrm{noise}}$ & DMM+KDE & PCA+KNN & RF \\
\midrule
0.0 & 0  & 0.983833 / 0.983898 & 0.984667 / 0.984712 & 0.983000 / 0.983022 \\
0.5 & 8  & 0.983167 / 0.983246 & 0.982333 / 0.982399 & 0.982833 / 0.982909 \\
1.0 & 15 & 0.979000 / 0.979101 & 0.980000 / 0.980090 & 0.980000 / 0.980060 \\
2.0 & 30 & 0.983667 / 0.983700 & 0.985833 / 0.985864 & 0.986167 / 0.986194 \\
\bottomrule
\end{tabular}
\end{table}

\subsection{Experiment S4: Class imbalance within the DMM decision layer}
\label{subsec:exp_s4}

Experiment~S4 focuses \emph{exclusively} on the proposed DMM pipeline and is designed as an ablation study to clarify the role of
class imbalance \emph{inside DMM} rather than as a comparison against external baselines.
More precisely, we keep the DMM spectral embedding fixed and only modify the decision layer by contrasting a pure
maximum-likelihood (ML) rule with a prior-weighted maximum-a-posteriori (MAP) rule, under increasing imbalance of the binary class priors.
We report balanced accuracy and macro-F1, since both metrics emphasize minority-class performance and are therefore sensitive to imbalance.

Table~\ref{tab:s4_results_compact} highlights a sharp transition: while the ML decision rule remains highly stable and preserves
balanced accuracy close to $1$ even under severe imbalance (up to $95/5$), the MAP rule deteriorates rapidly as the prior becomes
more skewed, eventually collapsing to near-random balanced accuracy ($\approx 0.5$) in the extreme $95/5$ regime.
This behavior is expected: MAP is optimal for minimizing the overall misclassification risk under a strongly imbalanced sampling distribution,
but it tends to favor the majority class and may severely penalize minority recall, which is precisely captured by balanced accuracy and macro-F1.
Finally, in this experiment the count-based and class-normalized DMM operators yield identical embeddings and outcomes; for clarity,
Table~\ref{tab:s4_results_compact} reports the count-based results.

\begin{table}[H]
\centering
\caption{Experiment S4 (DMM-only, class imbalance): balanced accuracy and macro-F1 for ML versus MAP decision rules under different binary class priors.
Each entry is reported as \texttt{balanced-acc / macro-F1}.}
\label{tab:s4_results_compact}
\setlength{\tabcolsep}{6pt}
\renewcommand{\arraystretch}{1.12}
\begin{tabular}{c c c c}
\toprule
$p_{\mathrm{major}}$ & $p_{\mathrm{minor}}$ & DMM (ML) & DMM (MAP) \\
\midrule
0.50 & 0.50 & 0.995895 / 0.995888 & 0.995094 / 0.995109 \\
0.80 & 0.20 & 0.994351 / 0.991741 & 0.930661 / 0.953848 \\
0.90 & 0.10 & 0.993910 / 0.986778 & 0.788744 / 0.853848 \\
0.95 & 0.05 & 0.997170 / 0.982675 & 0.506990 / 0.500835 \\
\bottomrule
\end{tabular}
\end{table}

\begin{Remark}[Reproducibility]
All synthetic generators are fully specified by $(q,m_i,k,n)$ and class-conditional simplices $(p_{i,c})$,
and are available in the repository, as shown in the Section \ref{data_availability}, Data Availability.

This enables exact replication and systematic sweeps of regime parameters.
\end{Remark}

\section{Discussion and Limitations}
The density-matrix construction provides a principled operator view of supervised categorical learning in one-hot survey spaces.
The intrinsic rank bound $\rank(\rho)\le k$ yields a latent dimension that scales with the number of classes, suggesting stable embeddings even in very high ambient dimension.

Several limitations and open directions remain.
First, the square-root lifting $X=\sqrt{F}$ is a modeling choice to guarantee the use of the Hellinger distance and the Bhattacharyya affinity; alternative monotone transforms could be explored, potentially with theoretical guarantees under generative assumptions.
Second, KDE performance depends on bandwidth selection, particularly under severe imbalance or multimodal latent clusters.
Third, the current construction aggregates frequencies linearly within each class; extensions incorporating pairwise co-occurrences or interaction features may further improve expressivity at higher computational cost.

Finally, while this article focuses on synthetic validation to emphasize methodological behavior, real-world applications across domains (survey analytics, event-based prediction, discretized mixed-type problems) can be studied systematically in follow-up work.

\section{Conclusions and Future Work}
We presented a supervised dimensionality reduction and classification framework for categorical data based on a density-matrix-like operator induced by class-conditional frequencies.
The model yields a positive semidefinite, unit-trace matrix with intrinsic rank bounded by the number of classes, enabling compact spectral embeddings.
A latent-space probabilistic classifier was obtained via class-conditional KDE and maximum-likelihood prediction.
Structural invariances and complexity estimates clarify the mathematical robustness and computational advantages of the approach.

Future work will investigate transform variants beyond entrywise square roots, incorporate interaction-aware operators, and develop theoretical guarantees under explicit generative models for categorical blocks.
Extensive application-driven studies, including domain-specific datasets and interpretability analyses, are natural next steps and are best addressed in application-focused contributions.

\vspace{6pt} 





\section*{Author contributions}
Conceptualization, all authors; methodology, all authors; software, all authors; validation, all authors; formal analysis, all authors; investigation, all authors; resources, all authors; data curation, all authors; writing---original draft preparation, all authors; writing---review and editing, all authors; visualization, all authors; supervision, A.F; project administration, A.F; funding acquisition, A.F. All authors have read and agreed to the published version of the manuscript.

\section*{Funding}
    This research was funded by the grant number COMCUANTICA/007 from the Generalitat Valenciana, Spain, by the grant number INDI24/17 from the Universidad CEU Cardenal Herrera, Spain. Additionally, R. Bosch-Romeu and A. Falco gratefully acknowledges the Spanish Ministry of Science, Innovation and Universities under grant \textit{Ayudas para la Formación de Profesorado Universitario FPU 2022} (FPU 22/02810). Moreover, this work has been supported by the project TED2021-129347B–C21 funded by MCIN/AEI/10.13039/501100011033 and by the “European Union NextGenerationEU/PRTR”. This research was supported by a Fulbright Program grant sponsored by the Bureau of Educational and Cultural Affairs of the United States Department of State and administered by the Institute of International Education.

\section*{Data availability}
\label{data_availability}
The code and synthetic experiment protocols supporting the reported results are publicly available on GitHub and archived on Zenodo:
\begin{itemize}
    \item Repository: \url{https://github.com/afalco/dmm-synthetic-experiments}. 
    \item Zenodo (v1.0.0): \url{https://doi.org/10.5281/zenodo.18299399}. 
\end{itemize}
The archive includes the Jupyter notebook used to generate Experiments S1--S4 and a reproducible conda environment specification (\texttt{environment.yml}).

\section*{Declaration of generative AI and AI-assisted technologies in the writing process} 
During the preparation of this work the authors used Claude Sonnet 4.5 to improve the readability and language of the manuscript. After using this tool/service, the authors reviewed and edited the content as needed and take full responsibility for the content of the published article.

\section*{Conflicts of interest}
The authors declare no conflicts of interest.

\end{document}